\documentclass[runningheads]{llncs}

\usepackage{graphicx}

\begin{document}
\title{Multi-Class Anomaly Detection}

\author{Suresh Singh\orcidID{0000-0002-1509-0497}\and
Minwei Luo \and
Yu Li}
\authorrunning{S. Singh et al.}
%
\institute{Portland State University, Portland OR 97207, U.S.A.\\
\email{singh@cs.pdx.edu}\\
\url{http://www.cs.pdx.edu/~singh}}

\maketitle

\begin{abstract}
We study anomaly detection for the case when the {\em normal} class consists of more than one object category. This is an obvious generalization of the standard one-class anomaly detection problem. However, we show that jointly using multiple one-class anomaly detectors to solve this problem yields poorer results as compared to training a single one-class anomaly detector on all normal object categories together. We further develop a new anomaly detector called DeepMAD that learns compact distinguishing features by exploiting the multiple normal objects categories. This algorithm achieves higher AUC values for different datasets compared to two top performing one-class algorithms that either are trained on each normal object category or jointly trained on all normal object categories combined. In addition to theoretical results we present empirical results using the CIFAR-10, fMNIST, CIFAR-100, and a new dataset we developed called RECYCLE.

\end{abstract}

\section{Introduction}

\noindent To motivate this formulation of anomaly detection, let us consider several applications: imagine a roadside garbage container that automatically separates any discarded trash into recycles (glass, cans, plastic, etc.) or garbage; machines used to separate normal blood cells (monocytes, neutrophils, basophils, etc.) from bacteria in the blood (sepsis); generic classification problems such as separating all species of cats from any other animal; separating sounds of different mosquito species from any other sound; and many, many more. In all these examples, the normal class in fact consists of several different object categories.

It is trivial to see that one-class anomaly detection algorithms \cite{ruff2020review,hendrycks19iclr,goyal20icml,li18bigdata,golan18neurips}, \cite{akcay18accv,nguyen19icml,schlegl17icipmi,song17cin,ruff18icml} can be easily applied to this problem. In a typical one-class formulation, the classifier is only trained on in-distribution samples and it learns a probability density function $P$ that captures normal behavior. Points that then map to a low probability region are classified as anomalous. Generalizing to the case when $m$ different object categories are all considered normal, we can train $m$ separate one-class anomaly detectors $P_1, \cdots, P_m$ and use them together to classify any new object. However, we show that this approach is not as good as training a single one-class anomaly detector on all the $m$ categories combined. This result is demonstrated empirically as well as theoretically. The intuition is that errors in classification of each of the $m$ one-class detectors are cumulative thus resulting in poor joint performance.

Novelty detection \cite{marco2014sp,masana2018bmvc,tack2020corr} and out-of-distribution (OOD) detection \cite{hendrycks2017iclr,ren2019neurips,balaji2017nips,lee2018training}, \cite{liang18iclr,vyas2018eccv} are related to anomaly detection but our model has significant differences. A common model used in OOD is to consider, for example, the fMNIST data set as normal and MNIST as OOD or CIFAR-10 as normal and natural images as OOD, etc. This model for defining what constitutes OOD has recently come under criticism \cite{ahmed2020aaai,ren2019neurips}. As \cite{ahmed2020aaai} notes, {\em different datasets are created and curated differently} and thus it is likely that OOD algorithms are learning to identify these idiosyncrasies rather than meaningful features. Similarly, they argue that the context of the OOD task also matters -- what is OOD in one context may not be OOD in another. \cite{ren2019neurips} also argues for more realistic benchmarks for OOD (they use a bacteria genome database) observing that distributional shifts learned in traditional OOD tasks may have learned the background and thus have high error when used for realistic tasks. Finally, \cite{ahmed2020aaai} proposes testing OOD algorithms by sticking to the same dataset. They study the case when 9 out of 10 CIFAR-10 classes are normal and the tenth is treated as anomalous. We note that our model follows this approach. We study cases when 2, 5, 9 classes of CIFAR-10, MNIST, and fMNIST are normal while the others are anomalous.

In novelty detection, a data point is considered an outlier if it differs significantly from a collection of normal data points (e.g., a mammogram with a lump is an outlier in a collection of normal mammograms). The outlier and the normal data points can be seen as being similar and what distinguishes the outlier is its distance from a majority of normal points. A multi-class novelty model is an extension where the number of outliers is large enough to form a separate cluster but all the points are still similar. {\em In our multi-class model, on the other hand, the normal classes can be arbitrary} (e.g., the normal classes may be dog and truck images).

There are two variations of the multi-class anomaly detection problem -- the class labels of normal samples may either be unknown or known during training. We call these two formulations {\em inseparable} and {\em separable} respectively and we study both in this paper. The major results of our paper can be summarized as follows:
\begin{itemize}
\item We prove that training a one-class classifier in the inseparable case yields a \underline{higher} AUC (Area Under the Curve) value than training multiple single-class classifiers for the separable case and combining their outputs. Our proof uses the formulation for combining probabilities from the theory of {\em belief functions}, Lemma \ref{lemma1}.
    \item We illustrate these results using two recent one-class anomaly detection algorithms DROCC \cite{goyal20icml} and DeepSVDD \cite{ruff20iclr} for the task of multi-class anomaly detection. We did not use OOD algorithms such as \cite{hendrycks2017iclr} because, as noted previously \cite{ahmed2020aaai}, these algorithms train on one dataset and test on another and likely learn dataset idiosyncrasies rather than meaningful image features.
    \item We present a new algorithm called DeepMAD for the {\em separable case} and compare the accuracy of DeepMAD and two versions of each of DROCC and DeepSVDD (in one version we train $m$ separate one-class classifiers and in the other we train a single one-class classifier by combining training data from all $m$ normal classes). These five algorithms are compared using CIFAR-10, fMNIST, CIFAR-100, and RECYCLE. The RECYCLE dataset is one we created to study the problem of classifying recycles described in the first paragraph of this section. We show that DeepMAD performs well above all other algorithms by a wide margin.
    \end{itemize}

\section{Related Work}

A variety of one-class novelty detection algorithms exist \cite{zhang2020hybrid,masana2018bmvc,chen2020icml,mohammad2018cvpr,pidh2018neurips,abati2019cvpr} but, as noted previously \cite{perera2019cvpr1}, these algorithms are not directly applicable to the single-class anomaly detection problem. In the latter case we assign labels whereas in the former we learn a latent space that separates the normal class effectively from outliers. 
There have been a few additional works that consider a multi-class novelty detection problem \cite{perera2019cvpr2} where the dataset is split into two subsets of classes with one treated as normal. A large reference labeled dataset is used in conjunction with normal samples during training. Our problem formulation differs from this in that we do not use any reference dataset as a comparator to extract features. We also consider cases when the split of normal/anomaly classes is variable.


In OOD detection, out of distribution samples are detected via the basic predictive confidence of the classifier (MSP score) \cite{hendrycks2017iclr} or measures such as temperature-scaled softmax scores (ODIN) \cite{liang18iclr} or confidence calibration \cite{devries2018learning}. Approaches as in \cite{hendrycks19iclr} use other datasets as proxies for OOD examples while \cite{lee2018training} uses a GAN to generate negative examples for training. Other methods to improve the representation of in-distribution data includes using self-supervised approaches such as contrastive learning \cite{chen2020icml}, alternative training strategies such as margin loss \cite{vyas2018eccv}, or metric learning \cite{masana2018bmvc}. \cite{liron2020iclr} splits normal data into $M$ and learns features to separate them.

Recently, deep learning based methods have been used for anomaly detection. For instance \cite{ruff20iclr,ruff18icml} describes DeepSVDD in which a deep neural network is trained to map the in-class data into a sphere of minimal volume, while \cite{ghafoori2020sdm} maps to multiple spheres. DROCC \cite{goyal20icml} trains the network to distinguish manufactured out of distribution points that are perturbations of in-distribution points as a way to learn more discriminating features. Another approach has been to train auto encoders to learn a lower dimensional representation \cite{li18bigdata,nguyen19icml,schlegl17icipmi}. An entirely different way of looking at the problem is to use self-supervision by learning {\em transformations} where a new sample is classified as normal if transforms applied to it can be correctly identified \cite{golan18neurips,hendrycks19neurips,sohn2021iclr}. 



\section{Generalizing Single-Class Anomaly Detection to the Multi-Class Case}

\begin{table*}[th]
\centering
\caption{Probability estimates for sample $x\in {\cal X}$ provided by $m$ classifiers.}
\begin{tabular}{cc c c c  c c c c}
&\{1\} & \{2\} & $\cdots$ & $\{m\}$ & $\neg\{1\}$ & $\neg\{2\}$ & $\cdots$ & $\neg\{m\}$\\ \hline
$({\cal F}_1,D_1)$ & $P_1(1|x)$ & 0 & $\cdots$ & 0 & $1-P_1(1|x)$ & 0 & $\cdots$ & 0\\ \hline
. & . & . & $\cdots$ & . & . & . & $\cdots$ & . \\ \hline
$({\cal F}_m, D_m)$ & 0 & 0 & $\cdots$ & $P_m(m|x)$& 0&0&$\cdots$ & $1-P_m(m|x)$\\ \hline
\end{tabular}
\label{belief}
\end{table*}

Typically, anomaly detection methods learn a mapping ${\cal F}$ from the input space ${\cal X}$ to some lower dimensional feature space ${\cal Y}$. Inputs are then assigned a probability of being normal or anomalous via some function $D$. ${\cal F}$ is learned using the normal class examples while $D$ may be static (e.g., a distance measure of points in feature space) or jointly learned. If we consider using one-class anomaly detectors for the multi-class case where there are $1\leq i\leq m$ normal classes, for each class we learn one classifier giving us a set $\{({\cal F}_i,D_i)\}_{i=1}^m$ of one-class anomaly detectors. The question then is how to use these $m$ detectors for the multi-class case.

We can use the theory of {\em belief functions} in risk analysis \cite{clemen99} to provide a framework for combining these $m$ classifiers. In this theory, the assumption is that $m$ experts provide their own assessment of risk among a set of choices. In order to develop a final estimate of risk, these $m$ estimates are combined. In the standard formulation of the problem, as applied to our case, let $U=\{1,2,\cdots,m, \Lambda\}$ denote the set of $m$ normal classes (1 to $m$) and a catchall class labeled $\Lambda$ which represents all {\em anomalies}. A general classifier assigns a probability $P(u|x)$ for the probability that an image $x$ belongs to the {\em set of classes $u$}, where $u\in 2^U$ is an element of the powerset of $U$. For example if $u=\{1,5\}$ then this is the probability that $x$ belongs to either of the classes 1 or 5.

This general formulation simplifies considerably for our case because the one-class classifiers will only assign a non-zero probability to two instances, as illustrated in Table \ref{ds1}. In other words, $P_i(u|x) = 0$ for all cases except when $u=\{i\}$ or $u=\neg\{i\} = \{1,2,\cdots,i-1, i+1,\cdots,m,\Lambda\}$. In order to combine the predictions of the $m$ one-class classifiers to compute the probability that $x$ is anomalous, we need to determine $P(\Lambda|x)$. According to the Dempster-Schafer Theory \cite{dempster67,shafer76},
\begin{equation}
    P(\Lambda|x) = \frac{1}{K} \sum_{\neg\{1\}\cap\neg\{2\}\cap\cdots\neg\{m\}=\Lambda} \prod_{i=1}^m P_i(\neg\{i\}|x)
\end{equation}
which simplifies to,
\begin{equation}
    P(\Lambda|x) = \frac{1}{K} \prod_{i=1}^m P_i(\neg\{i\}|x)
    \label{ds1}
\end{equation}
since there is just one case when the intersection of the predictions of the $m$ classifiers yields $\Lambda$. $K$ is given by,
\begin{equation}
    K = 1 - \sum_{u_1\cap u_2\cap\cdots u_m=\phi} \prod_{i=1}^m P_i(u_i|x)
    \label{ds2}
\end{equation}
where $u_i\in \{\{i\}, \neg\{i\}\}$ (because all other probabilities are zero from Table \ref{belief}, these are the only non-trivial cases). The above equation simplifies to,
\begin{equation}
    K = \prod_{i=1}^m P_i(\neg\{i\}|x) + \sum_{i=1}^m P_i(\{i\}|x) \prod_{j=1,j\neq i}^m P_j(\neg\{j\}|x)
\end{equation}
From the above formulation it follows that an input $x$ is classified as normal if {\em any} of the $m$ classifiers say it is normal and it is classified as anomalous if {\em all} classifiers classify it as anomalous. Thus, it is easy to construct a multi-class anomaly detection algorithm using $m$ one-class anomaly detectors as shown in Algorithm 1.

\medskip
\noindent
{\bf Algorithm 1:} (Using $m$ Single-Class Anomaly Detectors) 

\noindent
{\em Training:} Let $({\cal F}, D)$ be any one-class anomaly detection algorithm. Given training set $X = \{X_1, X_2, \cdots, X_m\}$ consisting of training examples from $m$ classes $X_i\subset {\cal X}_i$, train $({\cal F},D)$ separately on each class producing $m$ classifiers, $({\cal F}_i,D_i) 1\leq i\leq m$.

\noindent
{\em Testing:} 
Given $x\in {\cal X}$, classify it with each of the $m$ classifiers. {\em Declare $x$ anomalous if all classifiers classify it as anomalous.}
\medskip




\noindent
Another approach for training a one-class classifier for the $m$ class case is to simply combine the training data from all $m$ classes and treat that as a single class. Doing so gives us {\em Algorithm 2}. 
\begin{lemma}
Algorithm 2 has a higher AUC value than Algorithm 1 when the same one-class anomaly detection algorithm is used in both cases.
\label{lemma1}
\end{lemma}

\begin{proof}
Recall that the AUC value is the integral of the True Positive Rate (TPR) vs False Positive Rate (FPR) curve (each point on the curve corresponds to a different value for the detection threshold $T$). For Algorithm 1 we can write the TPR and FPR as,
\[
\begin{array}{lcl}
\mbox{TPR}_1 & = & 1 - \mbox{False Negative Rate} \\
\multicolumn{3}{c}{=  1-\frac{1}{K}\sum_{x_j \in \cup_{i=1}^m {\cal X}_i} p(x_j)\prod_{i=1}^m (1-P_i(\{i\}|x_j)) 
}\\
\mbox{FPR}_1 &=& 1-\mbox{True Negative Rate} \\
\multicolumn{3}{c}{= 1-\frac{1}{K}\sum_{x_j \in {\cal X}\setminus\cup_{i=1}^m {\cal X}_i} p(x_j)\prod_{i=1}^m (1-P_i(\{i\}|x_j))}\\
\end{array}
\]
Note that the difference in the two expressions is in the sets that $x_j$ is selected from. The similar expressions for Algorithm 2 are,
\[
\begin{array}{lcl}
\mbox{TPR}_2 & = & 1 - \mbox{False Negative Rate}\\
\multicolumn{3}{c}{= 1-\frac{1}{K}\sum_{x_j \in \cup_{i=1}^m {\cal X}_i} p(x_j)(1-P(\mbox{Normal}|x_j))} \\
\multicolumn{3}{c}{ =  1-\frac{1}{K}\sum_{x_j \in \cup_{i=1}^m {\cal X}_i} p(x_j)\sum_{i=1}^m \frac{1}{m}(1-P_i(\{i\}|x_j))}\\
\mbox{FPR}_2 &=& 1-\mbox{True Negative Rate}  \\
\multicolumn{3}{c}{=  1-\frac{1}{K}\sum_{x_j \in {\cal X}\setminus\cup_{i=1}^m {\cal X}_i} p(x_j)(1-P(\mbox{Normal}|x_j)} \\
\multicolumn{3}{c}{ = 1-\frac{1}{K}\sum_{x_j \in {\cal X}\setminus\cup_{i=1}^m {\cal X}_i} p(x_j)\sum_{i=1}^m \frac{1}{m}(1-P_i(\{i\}|x_j))}\\
\end{array}
\]
The difference in the TPR (and FPR) values of this set of expressions is that for Algorithm 1 we take a product of the form $\prod_i (1-P_i(\{i\}|x_j))$ whereas for Algorithm 2 this is a mean of the same probabilities. As $m$ increases, we would expect the product term to decrease rapidly. Indeed, the values of TPR and FPR for Algorithm 1 approach 1 as $m$ increases resulting in a purely random classifier. On the other hand, the TPR and FPR for Algorithm 2 remain relatively unchanged since they use an arithmetic mean of the probability rather than a geometric mean as in Algorithm 1.
\end{proof}

\noindent {\bf Corollary 3.1.1:}
As $m$ increases, the AUC value for  Algorithm 1 decreases.
\label{corr1}

\begin{proof}
This is easy to see since we note that the FPR and TPR values approach 1 as $m$ increases.
\end{proof}
Note that this corollary is not true for Algorithm 2.
\medskip

\noindent
{\bf Algorithm 2:} (Training a {\em single} classifier by combining all $m$ classes) 

\noindent
{\em Training:} Given a training set $X = X_1\cup X_2\cup\cdots\cup X_m$, train a single classifier $({\cal F},D)$ on this set (see {\em Notes} in Algorithm 1).

\noindent
{\em Testing:} An example $x$ is classified as normal or anomalous by this classifier (just like any single-class classifier).
\medskip




\subsection*{Algorithms Studied}

In this paper, we use two recent and high performing single-class anomaly detection algorithms -- DROCC \cite{goyal20icml} and DeepSVDD \cite{ruff20iclr}. For each algorithm we consider two variations corresponding to Algorithms 1 \& 2 above. Thus, DROCC($m$) and DeepSVDD($m$) correspond to Algorithm 1 while DROCC and DeepSVDD correspond to Algorithm 2. We use the author provided code for our experiments.

\section{DeepMAD: Deep Multi-class Anomaly Detection}

Our algorithm DeepMAD is based on Algorithm 1 above, but with important changes. The key insights we applied are the following:
\begin{itemize}
    \item For any single-class algorithm, AUC values increase if the {\em variance} of the pdf of the normal class is reduced because it increases the TPR (True Positive Rate) without affecting the FPR (False Positive Rate), assuming the means of the normal and anomaly classes do not shift. 
    \item In the absence of anomalies, the features learned represent what is {\em common} among the examples. DROCC attempts to learn {\em discriminating} features by using the manufactured examples. Similarly, algorithms that add noise to training examples or that learn transformations etc. are all attempting to reduce variance by learning discriminating features..
    \item When moving to the multi-class case, observe that Algorithm 2 will likely learn features common among the normal classes in addition to features common to examples within each class, potentially resulting in larger variance.
\end{itemize}
In DeepMAD, we train $m$ one-class anomaly detectors. However, rather than train $P_i$ only on the $i$th normal class, we train it on all $m$ classes where we use the training examples from classes $\neg\{i\}$ as anomalous. 

\medskip
\noindent
{\bf Algorithm 3: DeepMAD}

\noindent
{\em Training:} Randomly initialize $m$ autoencoders $A_i$; 
For every $A_i$, train $A_i$ on provided examples $X_i$; 
Using the encoder part $E_i$ of the autoencoder, identify a point $c_i$, the "center" for this class;

\noindent
For every encoder $E_i$, create {\em labeled} training data $\{(x,l)|$ if $x\in X_i$ then $l=+1$ else if $x\in \bigcup_{j=1,j\neq i}^m X_j$ then  $l=-1\}$. Then train $E_i$ on this data using loss function ${\cal L}$

\noindent
{\em Result:} $m$ trained encoders $E_i$

\noindent
{\em Testing:} Given $x$ to classify, 
compute $d(x) = \min_{i=1}^m ||c_i - E_i(x;\theta_i)||_2$; 
If $d(x)< \gamma$ then $x$ is {\em normal} else $x$ is anomalous
\medskip

 
  





The algorithm for training proceeds in two steps. For every class $i$, we first train an autoencoder $A_i$ using MSE between the original and reconstructed image as the loss. Next, we train only the encoder $E_i$ using the other $m-1$ classes as anomalies and class $i$ as normal. For this, we identify a "center" $c_i$ for this class and learn a representation that maps points in $X_i$ to a sphere around it and maps other points from the remaining classes far away. We use the following loss function,
\[
\begin{array}{l}
    {\cal L}(\theta_i) =  \frac{1}{N}\sum_{j=1}^{N_i} ||E_i(x_j;\theta_i)-c_i||^2 + \\
    \frac{\eta}{N}\sum_{x_k \in (\bigcup_j X_j) \setminus X_i} (\mbox{max }(0, \delta-||E_i(x_k;\theta_i)-c_i||_2)^2)\\
    + \frac{\lambda}{2}\sum_{l=1}^{L} ||W^l||_F^2
\end{array}
\]
where $N = |X_1| + \cdots + |X_m|$ is the total number of normal class samples and $N_i = |X_i|$ is the number of examples from normal class $i$. $\delta, \eta$ and $\lambda$ are hyperparameters where $\eta$ controls the weight given to the {\em constrastive loss} term, $\delta$ is a distance from the identified center beyond which we would like to map points not in $X_i$, and $\lambda$ is a regularization parameter. The first term above attempts to learn a small sphere about the center for normal points in class $i$ while the second term attempts to maximize the distance of the remaining normal points (all except points in $X_i$) far away. The ensemble method in \cite{vyas2018eccv} also trains multiple classifiers and uses a mean of their softmax outputs for OOD detection. Unlike out work, they train on one dataset and test on another (recall the criticism from \cite{ahmed2020aaai} which notes that detecting anomalies within the same dataset is much harder than between dastasets). It is also unclear how well their method would scale down to few normal classes (say 2). 

\section{Empirical Evaluation} \label{results}

\begin{table*}[h]
    \caption{AUC range for CIFAR-10 (10 repetitions for each).}
\begin{center}
    \begin{small}
    \begin{tabular}{lccc} \hline
    &2-in, 8-out & 5-in, 5-out & 9-in, 1-out \\ \hline
    DROCC &$0.4728\longleftrightarrow 0.7252$ & $0.4316\longleftrightarrow 0.7219$ & $0.4107  \longleftrightarrow 0.7146$ \\
    &$\pm 0.0119 \hspace{0.3cm}\pm 0.0081$ & $\pm 0.0257 \hspace{0.3cm}\pm 0.0039$ & $ \pm 0.0454  \hspace{0.3cm}\pm 0.0079$ \\ \cline{2-2}
    Outlier&$(0,8)\hspace{0.3cm} 0.8359 \pm 0.0117$&&\\ \hline
    DROCC($m$) &$0.4216 \longleftrightarrow 0.6912$ & $0.3806 \longleftrightarrow 0.7023$ & $0.3439  \longleftrightarrow 0.6896$\\ 
    &$\pm 0.0424\hspace{0.3cm} \pm 0.0188$ & $\pm 0.0047 \hspace{0.3cm}\pm 0.0648$ & $ \pm 0.1034 \hspace{0.3cm} \pm 0.0453$\\ \cline{2-2}
    Outlier & $(0,8) \hspace{0.3cm}0.8255\pm 0.0137$\\ \hline
    DeepSVDD &$0.4088\longleftrightarrow 0.7623$ & $0.3382 \longleftrightarrow 0.7105$& $0.3058  \longleftrightarrow 0.6844$\\
    &$\pm 0.0068 \hspace{0.3cm}\pm 0.0193$ & $\pm 0.0076 \hspace{0.3cm}\pm 0.0077$& $ \pm 0.0169 \hspace{0.3cm}\pm 0.0144$\\ \hline
    DeepSVDD($m$) &$0.4147 \longleftrightarrow 0.7516$ & $0.3482\longleftrightarrow 0.6909$ & $0.3580 \longleftrightarrow 0.5864$\\
    &$\pm 0.0129 \hspace{0.3cm}\pm 0.0093$ & $\pm 0.0123\hspace{0.3cm} \pm 0.0133$ & $\pm 0.0166\hspace{0.3cm} \pm 0.0167$\\ \hline
     DeepMAD & {\bf 0.5396$\longleftrightarrow$ 0.7647} & {\bf 0.4929 $\longleftrightarrow$ 0.7738} & {\bf 0.5437 $\longleftrightarrow$ 0.7230}\\
    & $\pm 0.0031 \hspace{0.3cm}\pm 0.0014$ & $\pm 0.0046 \hspace{0.3cm}\pm 0.0022$ & $\pm 0.0028\hspace{0.3cm} \pm 0.0084$\\ \hline
    \end{tabular}
    \end{small}
\label{cifar10}
\end{center}
\end{table*}

We compare the performance of five algorithms -- DROCC, DROCC($m$), DeepSVDD, DeepSVDD($m$), and DeepMAD using four data sets: CIFAR-10, fMNIST, RECYCLE, and CIFAR-100. For DROCC and DeepSVDD we used the code provided by the authors while for the $m$ normal classes case we simply replicated the code and modified the main() function appropriately. For all the experiments with DROCC, we used parameters radius $r=8$, $\mu=1$, learning rate = 0.001, ascent step = 0.001 and the Adam optimizer. The reason for this choice is that the ablation study reported in \cite{goyal20icml} shows a fairly stable AUC value for this parameter setting across classes (though the specific value for $r$ that achieves optimal AUC for the single-class case can vary by class). The parameter settings for DeepSVDD used $\eta=1$, learning rate of 0.0001 and the Adam optimizer as well. We consider three multi-class anomaly detection cases for CIFAR-10 and fMNIST -- (2/8), (5/5), and (9/1) which correspond to 2, 5, 9 normal classes respectively. For CIFAR-100 we used the 20 super-classes and studied the (2/18) case. For the RECYCLE dataset (described below) we have 5 classes of recyclables and one that is assorted trash. We considered two cases: in one, we study the (4/1) case using only the recycles; we then study the case when the 5 recycles are all normal and the trash is the anomaly. In all cases we repeated the experiments with random seeds and report the 95\% confidence intervals of the achieved AUC values.
\begin{table*}[hbt!]
\caption{AUC range for fMNIST (10 repetitions each).}
\begin{center}
    \begin{small}
    \begin{tabular}{lccc} \hline
    &2-in, 8-out & 5-in, 5-out & 9-in, 1-out \\ \hline
    DROCC &$0.6873\leftrightarrow 0.9774$ & $0.5738\leftrightarrow 0.9260$ & $ 0.5408 \leftrightarrow 0.8247$ \\
    &$\pm 0.0937\pm 0.0049$ & $\pm 0.0397 \pm 0.0307$ & $ \pm  0.0961\pm 0.0507$ \\ 
    Mean & 0.8161 & 0.7448 & 0.6992 \\ \hline
   Deep &$0.6622\leftrightarrow 0.9871$ & $ 0.5438\leftrightarrow 0.9279$& $ 0.4551\leftrightarrow 0.8825$ \\
    SVDD&$\pm 0.0502\pm 0.0033$ & $\pm 0.0274 \pm 0.0325$& $ \pm 0.0285 \pm 0.0137$ \\ 
    Mean & {\bf 0.8538} & 0.7269 & 0.6523  \\\hline
   Deep & $0.6434\leftrightarrow 0.9714$ & $ 0.5732\leftrightarrow 0.8832$ & $0.4860 \leftrightarrow 0.9395$\\
    MAD& $\pm 0.0640\pm 0.0011$ & $\pm 0.0485\pm 0.0137$ & $\pm 0.0267 \pm 0.3466$ \\ 
    Mean & 0.8329 & {\bf 0.7739} & {\bf 0.7613}\\\hline
    \end{tabular}
    \end{small}
\end{center}
\label{auc-fmnist}
\end{table*}

\begin{table*}[hbt!]
\caption{AUC range for RECYCLE, and CIFAR-100 (10 repetitions each).}
\begin{center}
    \begin{small}
    \begin{tabular}{lcc|c} \hline
    & \multicolumn{2}{c}{RECYCLE} & CIFAR-100\\\hline
    & 4-in, 1-out &5-in, 1-out& 2-in, 18-out\\ \hline
    DROCC &0.4447 $\leftrightarrow$ 0.7997&0.9056&$0.3548 \leftrightarrow 0.7329$\ \\
    &$\pm 0.0176 \pm 0.0719$ &&$\pm 0.0006 \pm 0.0971$\ \\ 
    Mean &  0.6128 &&0.5638\\ \hline
   Deep &0.3703 $\leftrightarrow$ 0.8728 &0.9012&$0.4196\leftrightarrow 0.7185$\\
    SVDD&$\pm 0.0207 \pm 0.0079$ &&$\pm 0.0077 \pm 0.0180$\\ 
    Mean &  0.5791&&0.5559 \\\hline
   Deep &  0.5906 $\leftrightarrow$0.8283 &{\bf 0.9838}&$0.5384 \leftrightarrow 0.8213$\\
    MAD& $\pm 0.0035 \pm 0.0073$ &&$\pm 0.0018 \pm 0.0012$\\ 
    Mean & {\bf 0.6966} &&{\bf 0.6580}\\\hline
    \end{tabular}
    \end{small}
\end{center}
\label{auc-cifar100}
\end{table*}

\begin{figure*}[hbt!]
  \centering
  \includegraphics[width=.15\linewidth]{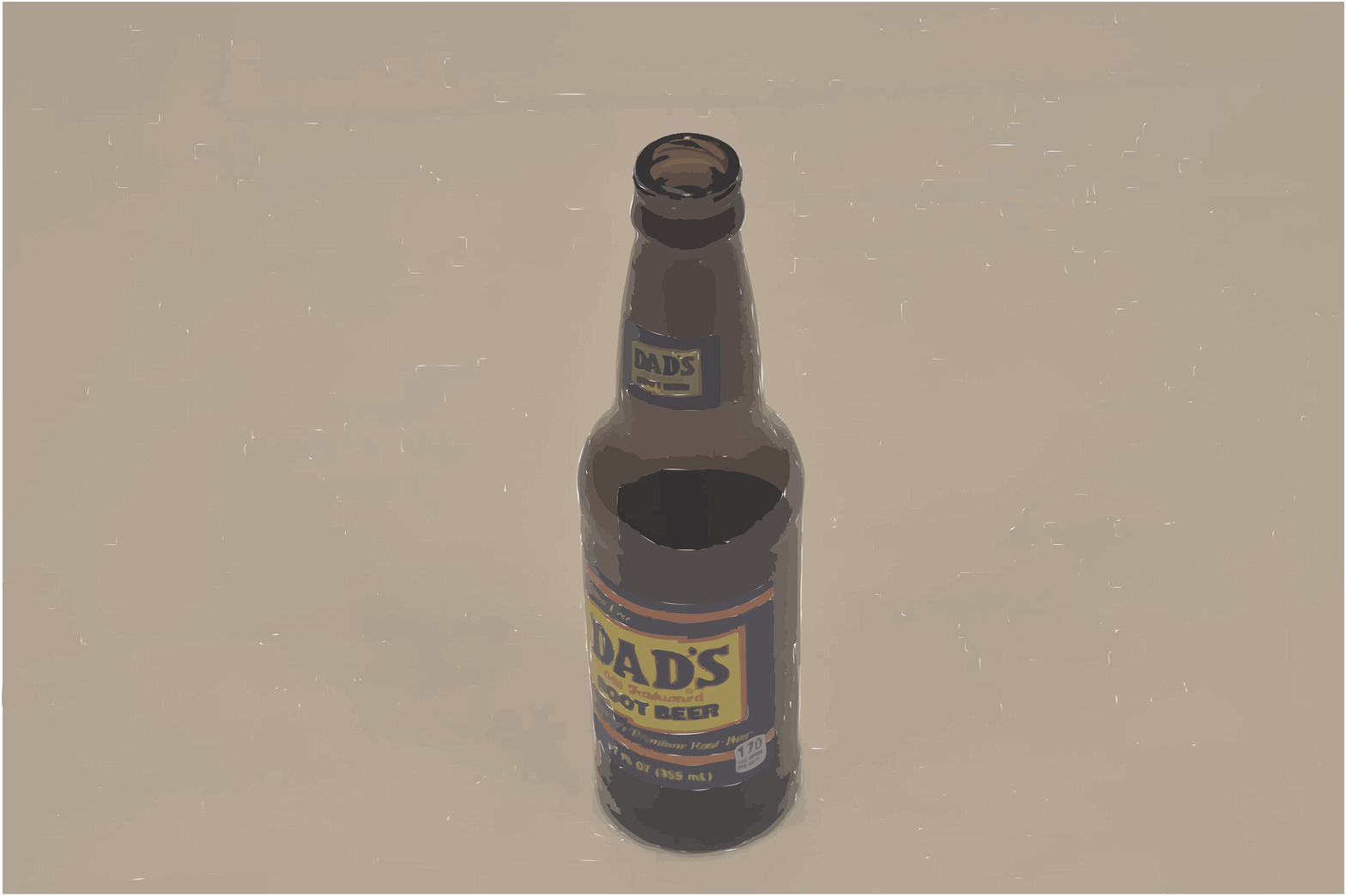}
  \centering
  \includegraphics[width=.15\linewidth]{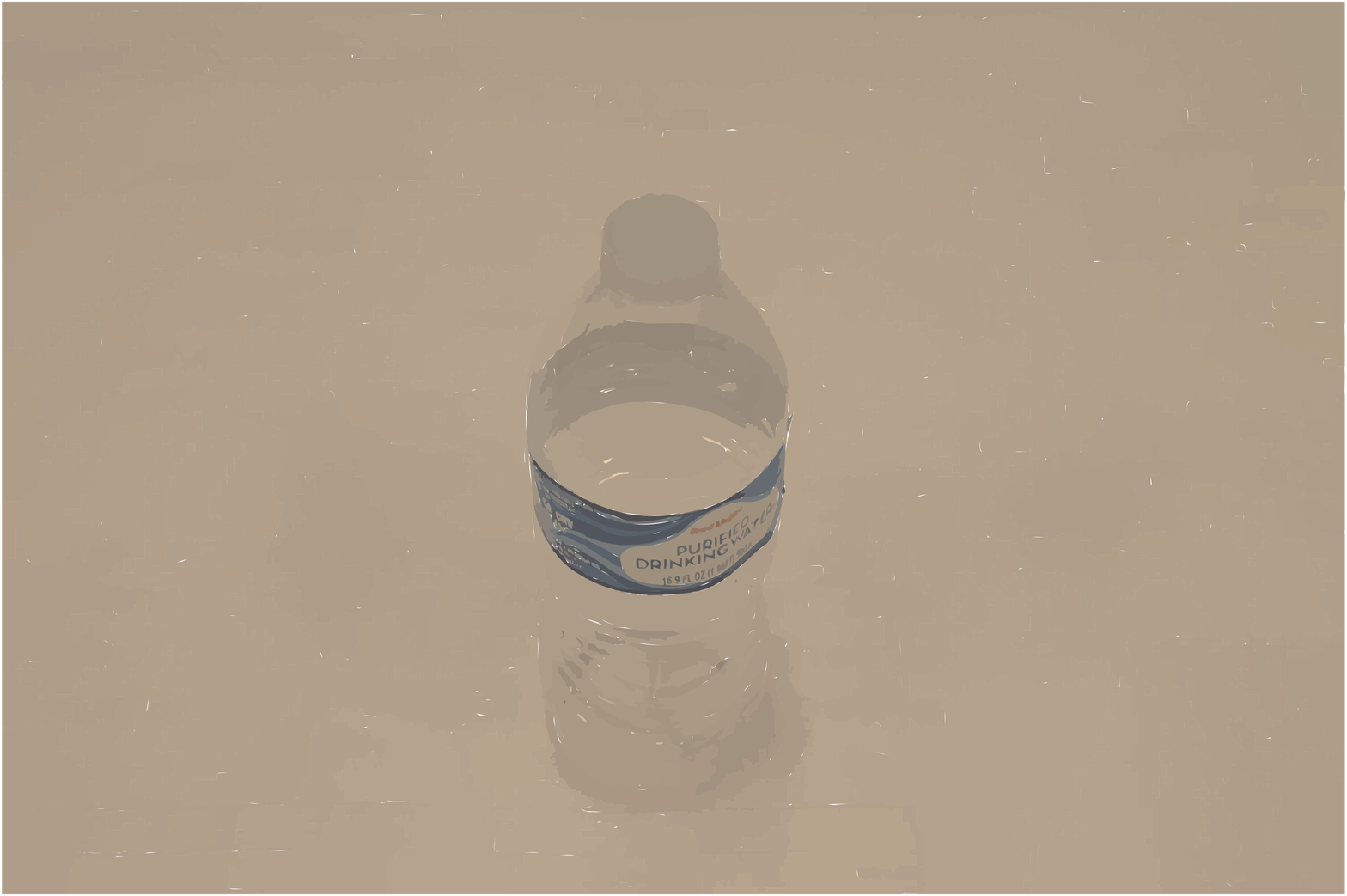}
  \includegraphics[width=.15\linewidth]{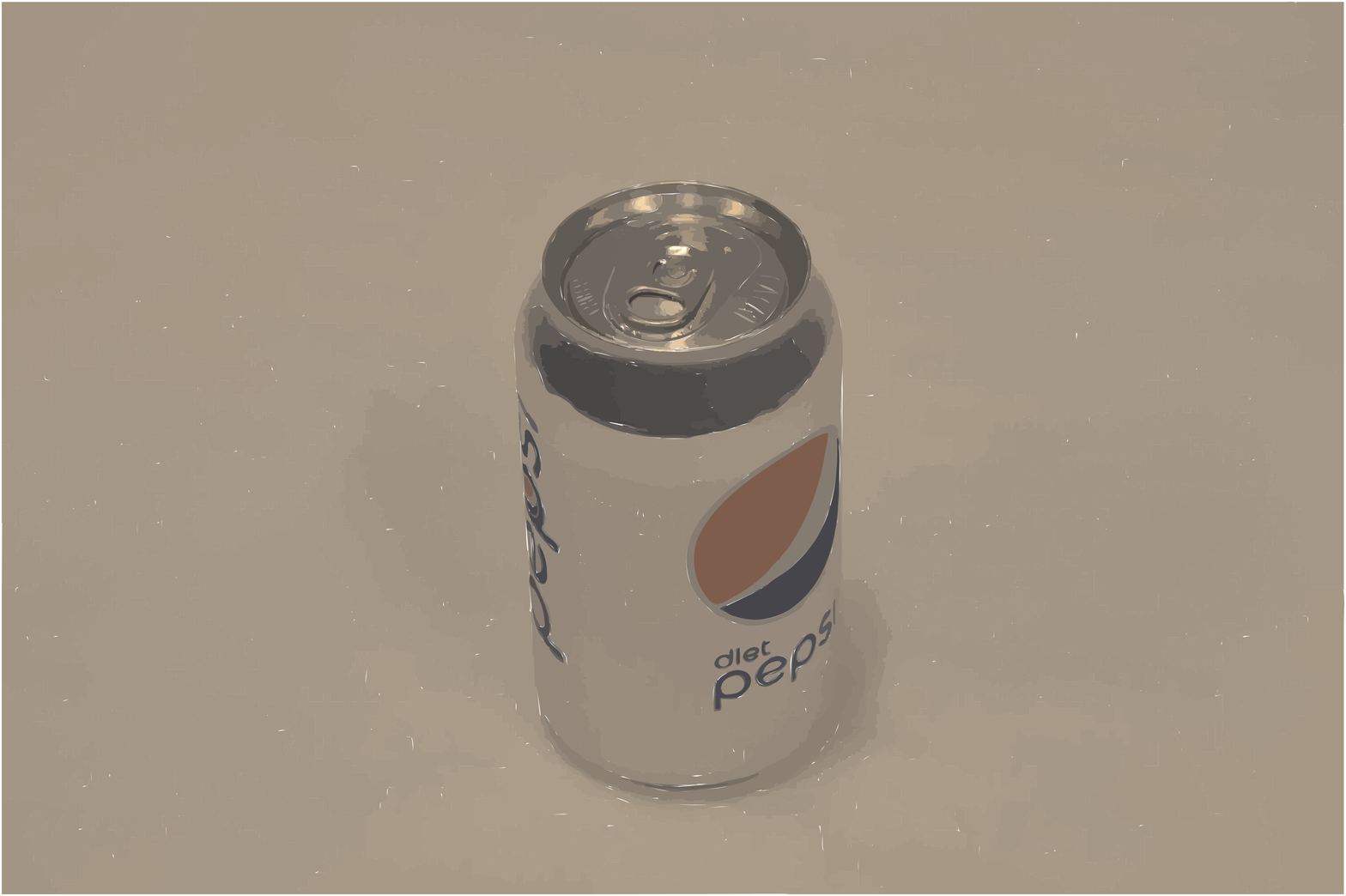}
  \centering
  \includegraphics[width=.15\linewidth]{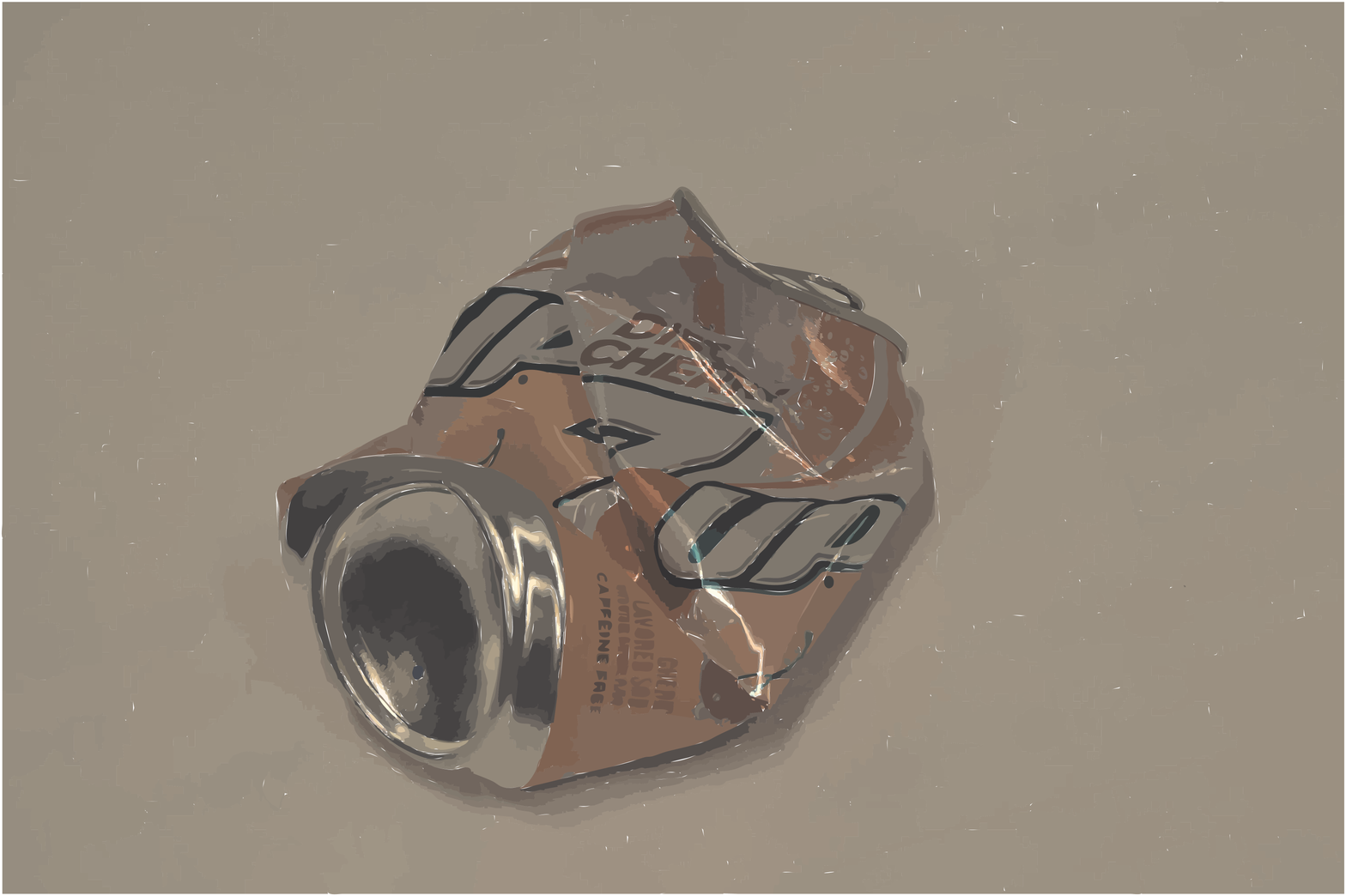}
  \centering
  \includegraphics[width=.15\linewidth]{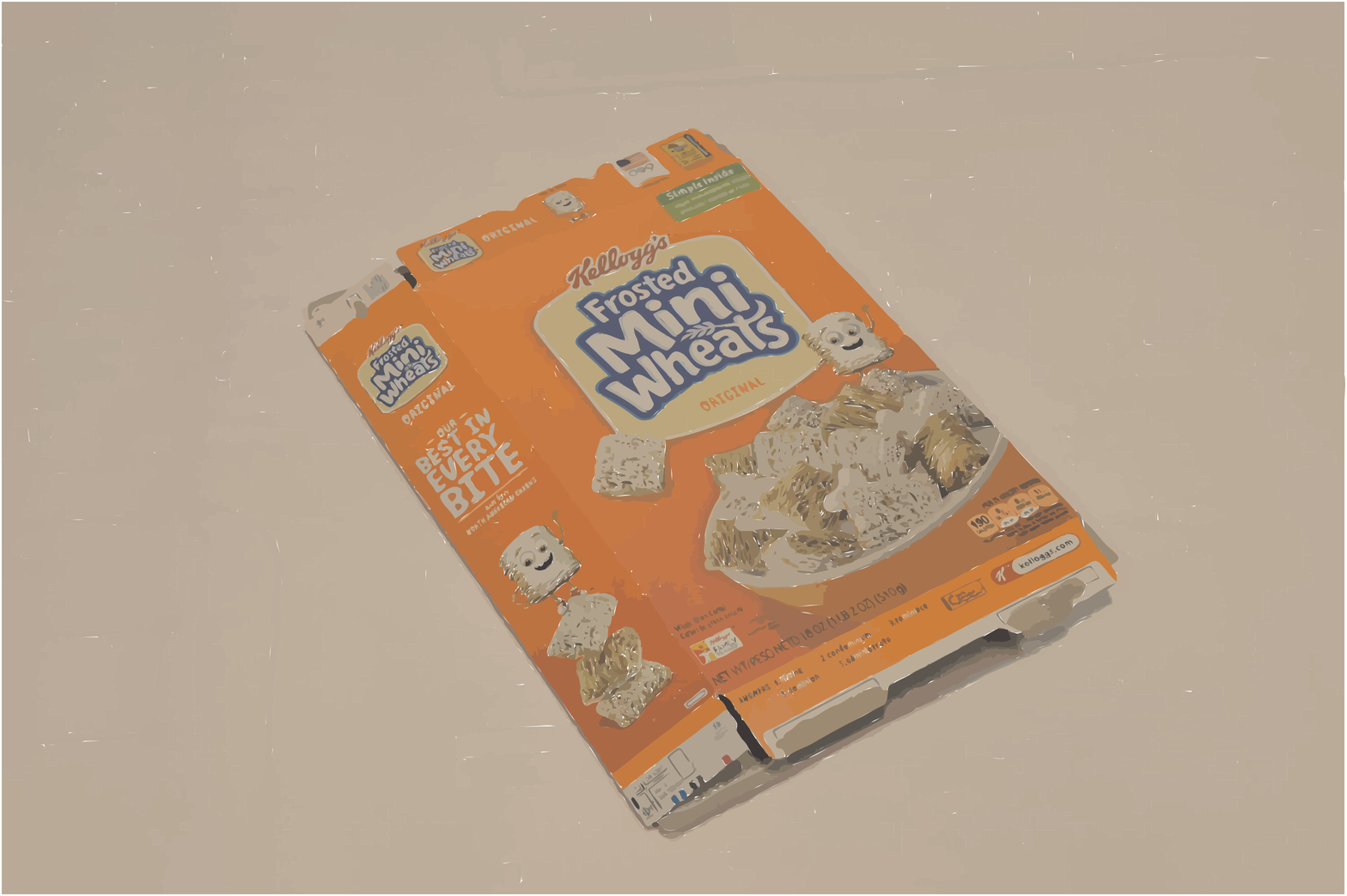}
    \caption{Examples of objects from the five (5) recycle classes.}
    \label{recycledataset}
\end{figure*}

\begin{figure*}[hbt!]
    \centering
  \centering
  \includegraphics[width=.15\linewidth]{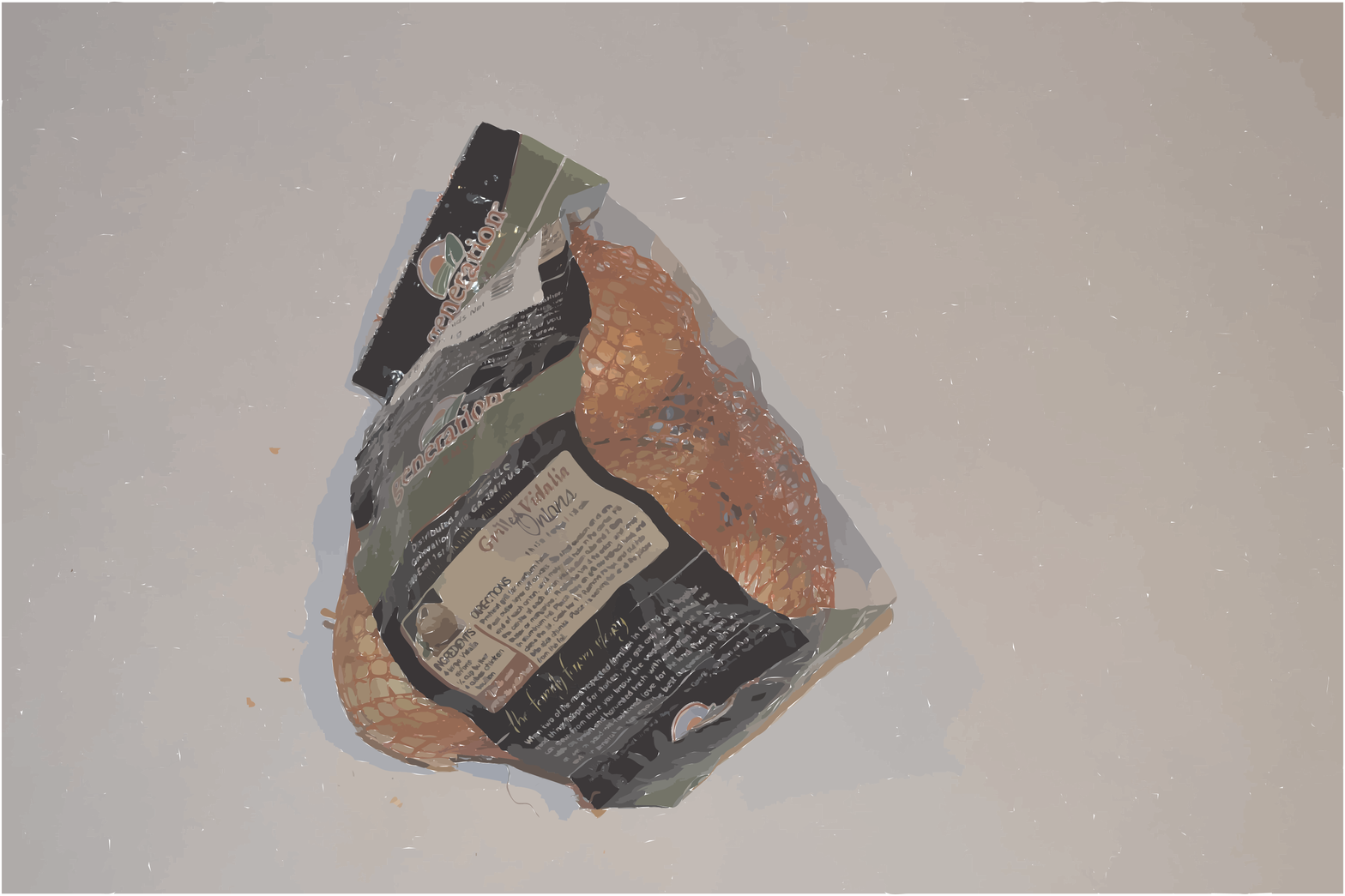}
  \centering
  \includegraphics[width=.15\linewidth]{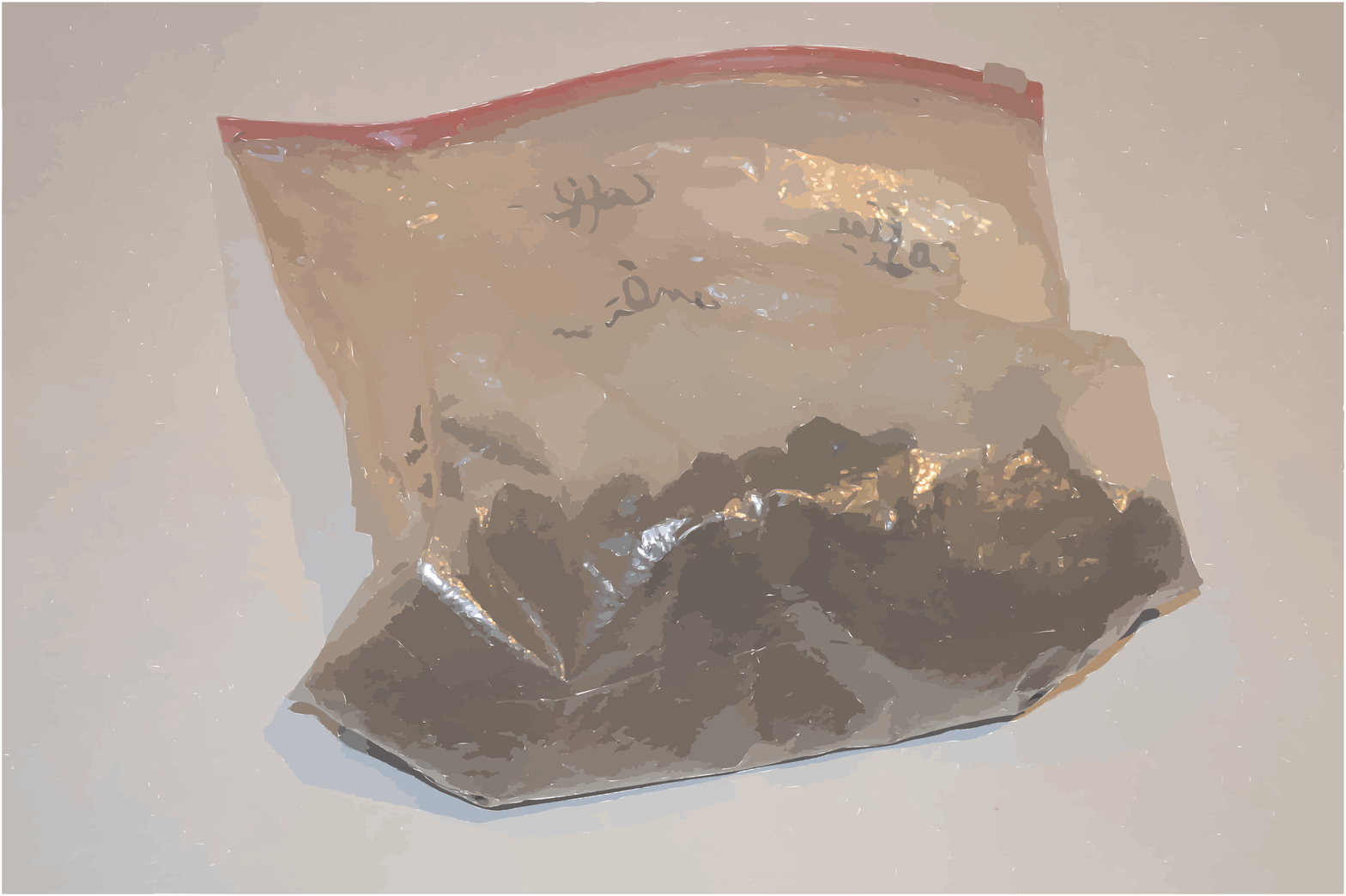}
  \centering
  \includegraphics[width=.15\linewidth]{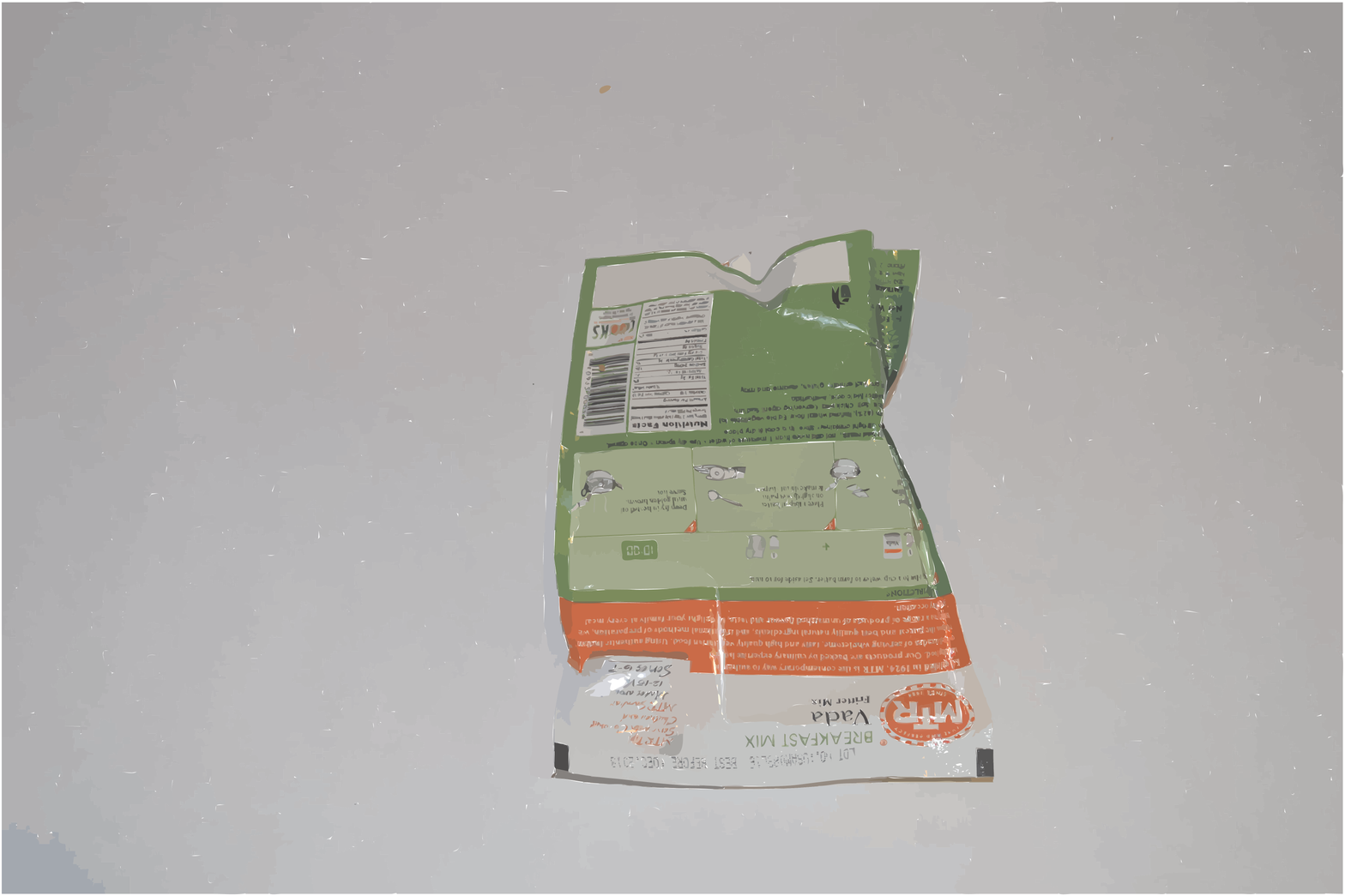}
  \centering
  \includegraphics[width=.15\linewidth]{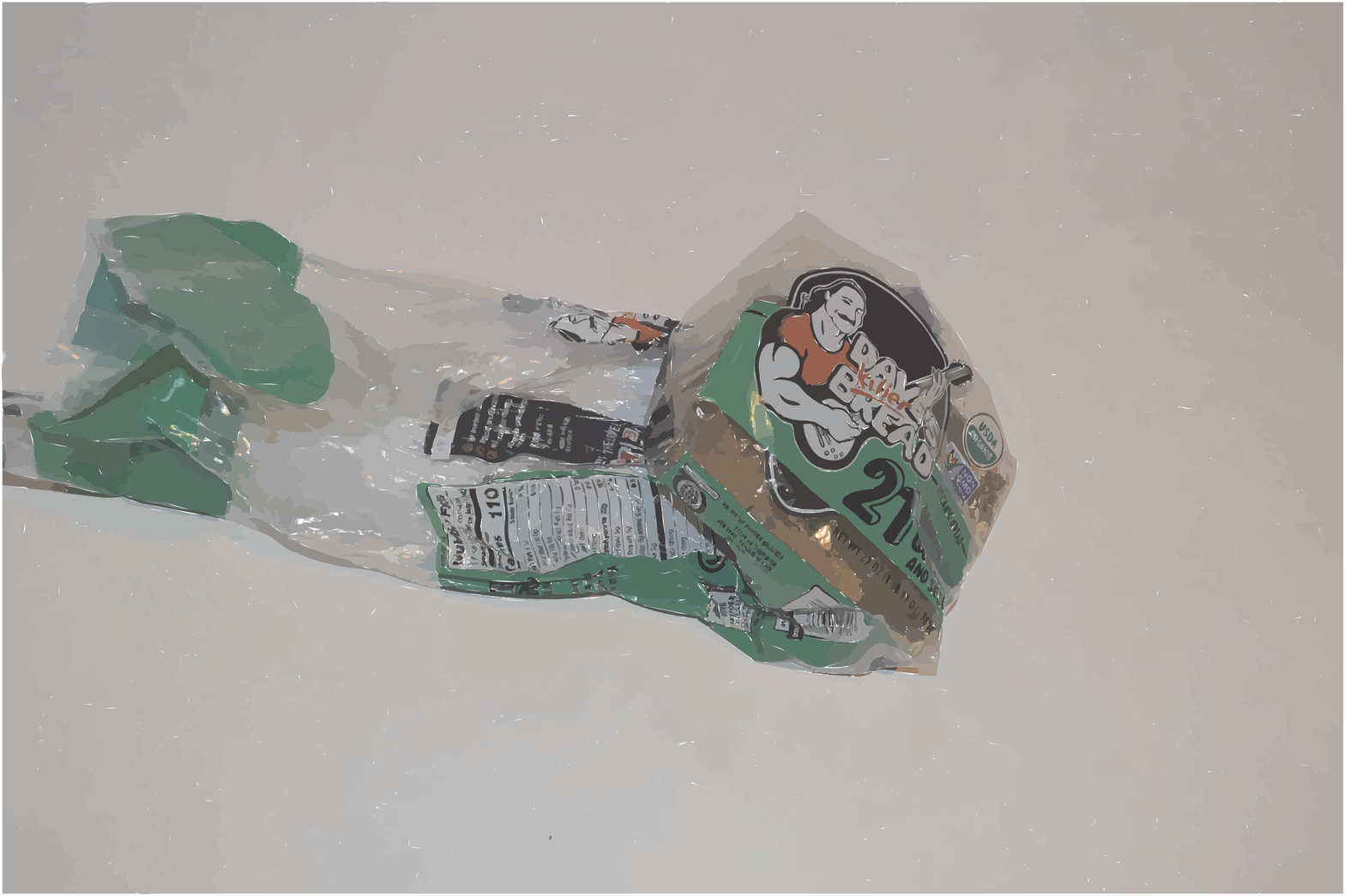}
  \centering
  \includegraphics[width=.15\linewidth]{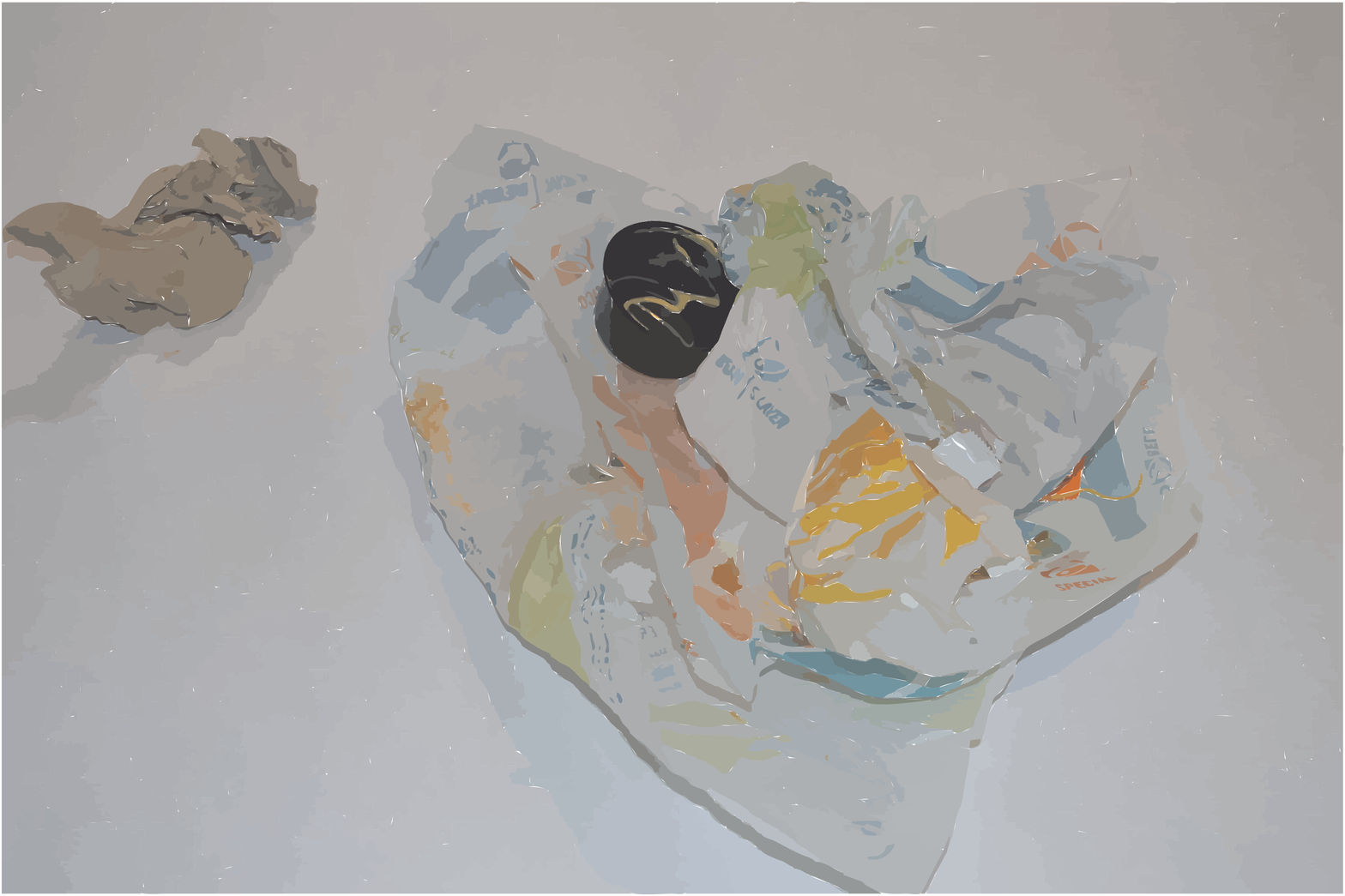}
    \caption{Examples of trash (images are all random trash and cannot be divided into classes).}
    \label{trashdataset}
\end{figure*}

The RECYCLE dataset contains 5 classes of recycles -- glass bottles, plastic bottles, cans, crushed cans, and cardboard boxes. The original images were 3008x2000x3 but we downsampled them to 32x32x3. There are a total of 11,000 images evenly divided among the five classes. 10,000 are used for training and 1,000 for testing. Samples of these images are shown in Figure \ref{recycledataset}. We also created 300 {\em trash} images to use for testing. Some examples are shown in Figure \ref{trashdataset}.


\begin{figure}[th]
\centerline
    {\includegraphics[width=0.5\columnwidth]{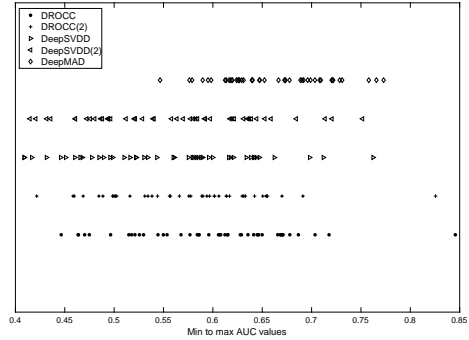}}
    \caption{2 in and 8 out case (CIFAR-10).}
        \label{cifar10-28}
\end{figure}

Table \ref{cifar10} summarizes the AUC values achieved by each of the five algorithms on CIFAR-10. For the (2/8) case, there are 45 combinations, for (5/5) there are 252 and for (9/1) there are ten. In the table we report on the range of values with their confidence intervals. For all three cases we note that DeepMAD has the best AUC values (with the exception of an outlier for DROCC). But more importantly, observe that DROCC performs better than DROCC($m$) and DeepSVDD performs better than DeepSVDD($m$), which is what the theoretical results from the previous section predicted. Figure \ref{cifar10-28} provides a scatter plot of all 45 combinations for the five algorithms for the case when two classes are normal. It is clear that which classes are considered normal has a huge impact on AUC values. For all but DeepMAD, some cases have an AUC value below 0.5 which is very poor. Its noteworthy that the case when classes 0 and 8 are normal DROCC has a very high AUC value (labeled an outlier in the table). Figure \ref{cifar10-55} plots the AUC values for the ten best and ten worst cases (out of 252) for when 5 classes are normal. The specific class combinations that have the best and worst values are different for the algorithms (details are in the supplementary materials). Aside from the fact that DeepMAD performs better, it also shows considerably low variance. This points to it learning very good discriminating features. 

We ran DROCC, DeepSVDD, and DepMAD on the fMNIST, RECYCLE, and CIFAR-100 datasets as well. We did not run DROCC($m$) and DeepSVDD($m$) because, as we have seen, training over all classes jointly performs better. Tables \ref{auc-fmnist}, \ref{auc-cifar100} provides AUC values for CIFAR-100 where we see that DeepMAD has about 10\% improvement over the other algorithms. In this case we considered the 20 super-classes and considered just a single case when 2 of the 20 classes are considered normal.  The tables also provides AUC values for fMNIST. Here the picture is more mixed. For the case when 2 classes are normal, all algorithms are very similar with DeepSVDD holding a 2\% edge when averaged over 45 combinations. DeepMAD performs better on average for the other two cases. For the RECYCLE dataset, DeepMAD performs well above the other two algorithms. The (4/1) case is when we only consider the five recycle classes and use four of them as defining normal with the fifth being the anomaly. We also considered the case when the five recycle classes were all normal and {\em trash} was considered as the anomaly. This is the (5/1) case in the table. All algorithms perform considerably better because trash images are very different from recycles. However, DeepMAD is about 8\% better than the other two.

\begin{figure}[th]
  \centerline
    {\includegraphics[width=0.5\columnwidth]{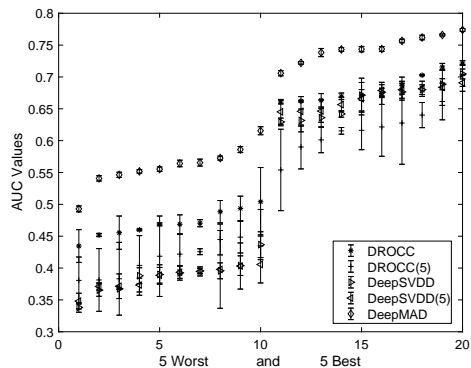}}
    \caption{5 in 5 out case (CIFAR-10).}
    \label{cifar10-55}
\end{figure}

\section{Conclusions}

The problem of multi-class anomaly detection occurs naturally in many settings, making it an important model to study. In this paper we adapt two one-class anomaly detection algorithms to the multi-class setting and develop a new algorithm DeepMAD. Our paper reports two significant results. First, we show that jointly learning a classifier (using single-class anomaly detection algorithms) for all the normal classes provides higher classification accuracy as compared to using several single-class classifiers. And second, we show that using supervised learning by treating $m-1$ out of $m$ normal classes as anomalous results in much compact representations and hence uniformly higher accuracy. We compare our algorithm against two recent single-class classifiers adapted to the multi-class scenario and show improvements in AUC values for almost all cases for CIFAR-10, fMNIST, and and CIFAR-100. However, we still see AUC values below 0.7 for many cases and addressing this constitutes a major thrust of future work. We also take a first step in analyzing the feature vectors as a way to understand where the algorithms fail.  We hope that the methodology we use for this analysis of feature vectors can be extended by others. Our code is available at \url{https://anonymous.4open.science/r/DeepMAD-16D4/}

\bibliographystyle{splncs04}
\bibliography{recycle-refs}

\end{document}